\def\year{2019}\relax
\documentclass[letterpaper]{article}
\usepackage{aaai}

\usepackage{times}
\usepackage{helvet}
\usepackage{courier}
\usepackage{algorithm}
\usepackage{algorithmic}
\usepackage{times}
\usepackage{amsmath}
\usepackage{amssymb}
\usepackage{xcolor}
\usepackage{soul}
\usepackage{graphicx}
\usepackage{mathtools}
\usepackage[utf8]{inputenc}
\usepackage[small]{caption}
\usepackage{amsthm}
\usepackage{footmisc}
\newtheorem{assumption}{Assumption}
\newtheorem{theorem}{Theorem}
\newtheorem{lemma}{Lemma}
\newtheorem{proposition}{Proposition}

\begin{document}
%
\title{Task Transfer by Preference-Based Cost Learning}
\author{
Mingxuan Jing\thanks{These two authors contributed equally}\footnotemark[2],
Xiaojian Ma\footnotemark[1]\footnotemark[2],
Wenbing Huang\footnotemark[3],
Fuchun Sun\footnotemark[2],
Huaping Liu\footnotemark[2]
\\
\footnotemark[2] Department of Computer Science and Technology, State Key Lab on Intelligent Technology and Systems,\\
National Lab for Information Science and Technology (TNList), Tsinghua University, Beijing 100084, China\\
\footnotemark[3] Tencent AI Lab, Shenzhen, Guangdong, China \\
\texttt{\{jmx16, maxj14\}@mails.tsinghua.edu.cn, \{fcsun, hpliu\}@tsinghua.edu.cn} \\
\texttt{hwenbing@126.com}
}

\maketitle
\begin{abstract}

The goal of task transfer in reinforcement learning is migrating the action policy of an agent to the target task from the source task. Given their successes on robotic action planning, current methods mostly rely on two requirements: exactly-relevant expert demonstrations or the explicitly-coded cost function on target task, both of which, however, are inconvenient to obtain in practice. In this paper, we relax these two strong conditions by developing a novel task transfer framework where the expert preference is applied as a guidance. In particular, we alternate the following two steps: Firstly, letting experts apply pre-defined preference rules to select related expert demonstrates for the target task. Secondly, based on the selection result, we learn the target cost function and trajectory distribution simultaneously via enhanced Adversarial MaxEnt IRL and generate more trajectories by the learned target distribution for the next preference selection. The theoretical analysis on the distribution learning and convergence of the proposed algorithm are provided. Extensive simulations on several benchmarks have been conducted for further verifying the effectiveness of the proposed method.

\end{abstract}

\section{Introduction}\label{intro}

Imitation Learning has become an incredibly convenient scheme to teach robots skills for specific tasks~\cite{wang2017robust,pathak2018zeroshot,daml,third,tcn,feeling}. It is often achieved by showing the robot various expert trajectories of state-action pairs. Existing imitation methods like MAML~\cite{MAML} and One-Shot Imitation Learning~\cite{One-shot} requires perfect demonstrations in the sense that the experts should perform the same as they expect the robot would do. However, this requirement may not always hold since collecting exactly-relevant demonstrations is resource-consuming.

One possible relaxation is assuming the expert to perform a basic task that is related but not necessary the same as the target task (sharing some common features, parts, etc). This relaxation, at the very least, can reduce the human effort on demonstration collecting and enrich the diversity of the demonstrations for task transfer. For example in Figure~\ref{flowchart}, the expert demonstrations contain the agent movements along an arbitrary direction, while the desired target is to move along only one specified direction.

\begin{figure}
\begin{center}
    \includegraphics[width=1.0\linewidth]{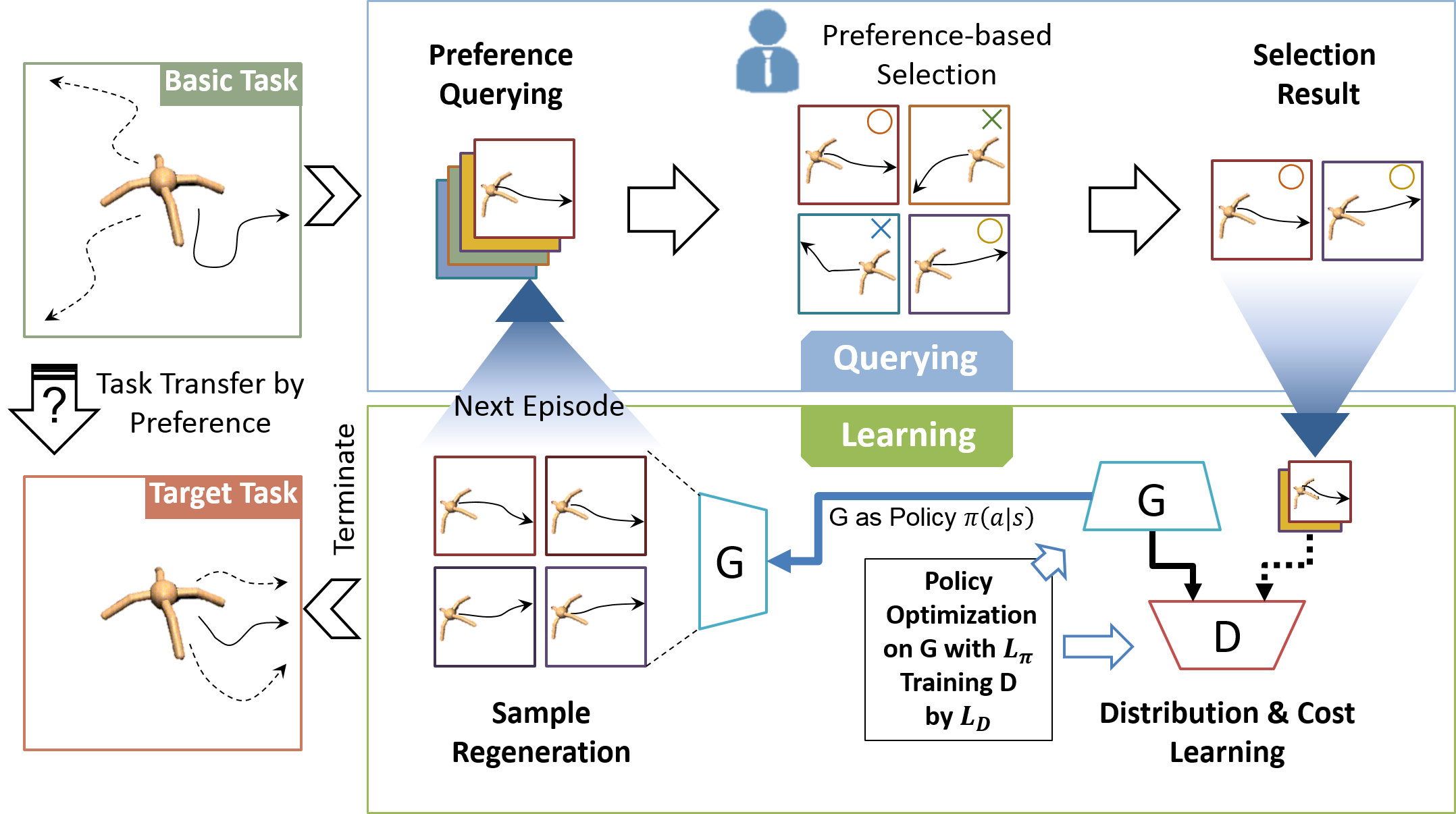}
\end{center}
   \caption{Problem statement and method introduction. As an example, we want to transfer a multi-joint robot from moving towards arbitrary directions (basic task) to moving forward (target task). Our preference-based task transfer framework iterate following two steps. 1. Querying expert for preference-based selection; 2. Learning distribution and cost simultaneously from selected samples, doing policy optimization and re-generating more samples, which would have the same distribution as the selected ones.}
\label{flowchart}
\end{figure}


Clearly, it does not come for free to learn target action policy from the relaxed expert demonstrations. More advanced strategies are required to transfer the action policy from the demonstrations to the target task. The work by~\cite{PBRecoSys} suggests that using experts' preference as a supervised signal can achieve nearly optimal learning result. Here, the preference refers to the highly-abstract evaluation rules or choice tendency of a human for making comparison and selection among data samples. Indeed, the preference mechanism has been applied in many other scenarios, such as complex tasks learning~\cite{PBRL_Survey}, policy updating~\cite{OpenAI_PBIRL}, and policy optimization combing with Inverse Reinforcement Learning (IRL)~\cite{PBIRL_PI} to name a few.


However, previous preference-based methods mainly focus on learning the utility function behind each comparison, where the distribution of trajectories is never studied. However, this would be inadequate for task transfer. The importance of modeling distribution comes from two aspects: 1. Learning the trajectory distribution takes a critical role in preference-based selection, which will be discussed lately; 2. With the distribution, it is more convenient to provide a theoretical analysis of the efficiency and stability of the task transfer algorithm (See Section~\ref{analysis}).

In this work, we approach the task transfer by utilizing the expert preference in a principled way. We first model the preference selection as a rejection sampling where a hidden cost is proposed to compute the acceptance probability. After selection, we then learn the distribution of the target trajectories based on the preferred demonstrations. Since the candidate demonstrations would usually be insufficient after selection, we augment the demonstrations with the samples of the current learned trajectory distribution and perform the preference selection and distribution learning iteratively. The distribution here acts as the knowledge which we make the transfer on. The theoretical derivations prove that it can improve the preference after each iteration and the target distribution will eventually converge. 

As the core of our framework, the trajectory distribution and cost learning are based on but has advanced the Maximum Entropy Inverse Reinforcement Learning (MaxEnt IRL)~\cite{MaxEntIRL} and its adversary version~\cite{CONN_GAN_IRL}. The MaxEnt IRL framework models the trajectory distribution as the exponential of the explicitly-coded cost function. Nevertheless in MaxEnt IRL, computing the partition function requires MCMC or iterative optimization, which is time-consuming and numerically unstable. Hence in adversary MaxEnt IRL, it avoids the computation of the partition function by casting the whole IRL problem into optimization of a generator and a discriminator. Although the adversary MaxEnt IRL is more flexible, it never delivers any form of the cost function, which is crucial for down-stream applications and policy learning. Our method enhances the original adversary MaxEnt IRL by redefining the samples from the trajectory level to the state-action level and devise the cost function using the outputs of the discriminator and generator. With the cost function, we can optimize the generator by any off-the-shelf reinforcement learning method and then the optimal generator could be used as a policy on the target task.

To summarize, our key contributions are as follow.
\begin{enumerate}
    \item We propose to perform imitation learning from related but not exactly-relevant demonstrations by making use of the expert preference-based selection.
    \item We enhance the Adversarial MaxEnt IRL framework for learning the trajectory distribution and cost function simultaneously.
    \item Theoretical analyses have been provided to guarantee the convergence of the proposed task transfer frameworks. Considerable experimental evaluations demonstrate that our method obtains comparable results with other algorithms that require accurate demonstrations or costs.
\end{enumerate}

\section{Preliminaries}\label{PreLim}

This section reviews fundamental conceptions and introduces related works to our method. 
Before further introduction, we first provide key notations used in this paper.

\noindent
\textbf{Notations.}
For modeling the action decision procedure of an agent, The Markov Decision Processes (MDP) without reward $ (\mathcal{S}, \mathcal{A}, \mathcal{T}, \gamma, \mu) $ is used, where $ \mathcal{S} $ denotes a set of states which can be acquired from environment; 
$ \mathcal{A} $ denotes a set of actions controlled by the agent; $ \mathcal{T}=p(s'| s, a) $ denotes the transition probability from state $s$ to $s'$ by action $a$;
$ \gamma \in [0, 1) $ is a discount factor; 
$ \mu $ is the distribution of the initial state $ s_0 $;  
$\pi (a|s)$ defines the policy.
A trajectory is given by the sequence of state-action pairs $\tau_i = \{(s^{(i)}_0, a^{(i)}_0), (s^{(i)}_1, a^{(i)}_1), \cdots\}$. We define the cost function parameterized by $ \theta $ over a s-a pair as $ c_\theta(a, s) $, and define the cost over a trajectory as $ C_\theta(\tau_i) = \sum_t c_\theta(a^{(i)}_t, s^{(i)}_t)$ where $ t $ is time step. A trajectory set is formulated by $n$ expert demonstrations, \emph{i.e.} $B_i = \{\tau_i\}_{i=1}^n$.

\subsection{MaxEnt IRL}\label{PreLimMaxEnt_IRL}
Given a demonstration set $B$, the Inverse Reinforcement Learning (IRL) method~\cite{IRL} seeks to learn optimal parameters $\theta$ of the cost function $C_\theta(\tau_i)$.
The solution could be multiple when using insufficient demonstrations. The MaxEnt IRL~\cite{thesis,relativeMaxEntIRL} handles this ambiguity by training the parameters to maximize the entropy over trajectories, leading to the optimization problem as: 
\begin{align}\begin{split}\label{OPTPROB}
                    &\max_{\theta} -\sum_{\tau}p(\tau)\log p(\tau) \\
    \text{s.t.}~~ & \mathbb{E}_{p(\tau)}[C_\theta(\tau_i)] = \mathbb{E}_{p_E(\tau)}[C_\theta(\tau_i)], \tau_i\in B,\\
                    & \sum_i p(\tau_i)= 1, \quad p(\tau_i)\geq0.
\end{split}\end{align}
Here $p(\tau)$ is the distribution of trajectories; $p_E(\tau)$ is the probability of the expert trajectory; $\mathbb{E}[\cdot]$ computes the expectation.
The optimal $ p(\tau) $ is derived to be the Boltzmann distribution associated with the cost $-C_\theta(\tau)$, namely,
\begin{align}\label{BoltzmannDist}
    p(\tau) = \frac{1}{Z} \exp(-C_\theta(\tau))\text{.}
\end{align}
Here $ Z $ is the partition function given by the integral of $ \exp(-C_\theta(\tau)) $ over all trajectories. 

\subsection{Generative Adversarial Networks}\label{PreLimGAN}

Generative Adversarial Networks (GANs) provides a framework to model a generator $ G $ and a discriminator $ D $ simultaneously.
$ G $ generates sample $ x \sim G(z) $ from noise $ z \sim N(0, I) $ , while $ D $ takes $ x $ as input, and outputs the likelihood value $ D(x) \in [0, 1] $ indicates whether $x$ is sampled from underlying data distribution or from generator~\cite{GAN}

\begin{align}\label{GANLOSS}\begin{split}
    \min_D \mathcal{L}_D =& \mathbb{E}_{x \sim p_{data}}[\log D(x)]\\
    & + \mathbb{E}_{z \sim N(0, I)}[\log (1-D(G(z))]\\
    \min_G \mathcal{L}_G =& \mathbb{E}_{z \sim N(0, I)}[-\log D(G(z))]\\
    & + \mathbb{E}_{z \sim N(0, I)}[\log(1-D(G(z))]\text{.}
\end{split}\end{align}

Generator loss $ \mathcal{L}_{G} $, discriminator loss $ \mathcal{L}_D $ and optimization goals are defined as~\eqref{GANLOSS}. Here $ \mathcal{L}_{G} $ is modified as the sum among logarithm confusion and opposite loss of $ D $ for keeping training signal in case generated sample is easily classified by the discriminator.

\section{Methodology}\label{Metho}

Our preference-based task transfer framework consists of 2 iterative sub-procedures: 1) querying expert preference and construct a selected trajectory samples set; 2) learning the trajectory distribution and cost function from this samples set for re-generating more samples for next episode. Starting from the demonstrations of the basic task, the trajectory distribution and cost function we learned are improved continuously. Finally, with the learned cost function, we can derive a policy of the target task.

The following sections will cover the modeling and analysis for all the two steps mentioned above. In Section~\ref{hidden_cost}, we will introduce the hidden-cost model for modeling the expert preference-based selection. Then in Section~\ref{MethoEAMaxEnt}, our enhanced Adversarial MaxEnt IRL for distribution and cost learning will be presented. We will combine the above two components to develop a preference-based task transfer framework and provide the theoretical analysis on it.

\subsection{Preference-based Sampling and Hidden Cost Model}\label{hidden_cost}

The main idea of our task transfer framework is transferring trajectory distribution with sample selection. Different from other transfer learning algorithms, the selection in our method only depends on preference provided by experts instead of any quantities. The preference of expert here could be abstract conceptions or rules on the performance of agents in target task, which are hard to directly be formalized as cost functions or provided numerically by the expert. In our preference-based cost learning framework, however, we only require experts to choose their most preferred samples among the given set generated on the last step, and try to use the selection result as the guidance on migrating the distribution from current policy to the target task policy.

We migrate the distribution by preference-based selection of samples in current set, the agent should be able to generate feasible trajectories on target task, which requires the probability of a trajectory on current task to be non-zero whenever the probability of that trajectory on target task is non-zero, and there should exist one finite value $M$ (which indicates the expected rejections made before a sample is accepted) that

\begin{align}\label{IncAssump}
    \forall \tau, \exists M \in (0, \infty)~\text{s.t.}~M p(\tau \in B_i) > p(\tau \in B_{tar})\text{,}
\end{align}
\noindent
where $ B_i $ and $ B_{tar} $ are feasible trajectory sets of current task and target task respectively. In previous section, we have shown that under MaxEnt IRL, the expert trajectories are assumed to be sampled from a Boltzmann distribution with negative cost function as energy. For an arbitrary trajectory $ \tau $, there will be

\begin{align}\label{BoltzmannDistTr}\begin{split}
    p(\tau &\in B_i) = p(\tau) = \frac{\exp(-C_i(\tau))}{Z_i} \propto \exp(-C_i(\tau))\\
    p(\tau &\in B_{tar}) = p_{tar}(\tau) = \frac{\exp(-C_{tar}(\tau))}{Z_{tar}}
    \propto \exp(-C_{tar}(\tau))\text{,}
\end{split}\end{align}

\noindent
where $C_i$ and $C_{tar}$ are ground truth costs over a trajectory of current and target task, while $c_i$ and $c_{tar}$ are corresponding cost functions. During selection, we suppose that the expert intends to keep the trajectory $\tau$ which have lower cost value on target task, which means the preference selection procedure could be seen as a rejection sampling over set $B_i$ with acceptance probability

\begin{align}\label{selectProb}\begin{split}
    p_{\text{sel}}(\tau) & = \frac{p_{tar}(\tau)}{M p_i(\tau)} = \frac{Z_i}{M Z_{tar}}\exp\left(C_i(\tau)-C_{tar}(\tau)\right) \\
    &\propto \exp\left(-C_{tar}(\tau)+C_i(\tau)\right)\text{.}
\end{split}\end{align}

We define the gap between target cost and current cost as \textit{hidden cost} $ c_h(s, a) = c_{tar}(s, a) - c_i(s, a) $ and for trajectory $C_h(\tau) = C_{tar}(\tau) - C_i(\tau)$. Thus we can view $C_h$ as a latent factor, or formally, a negative utility function~\cite{APRIL} that indicates the preference and at the same while indicates the gap between target distribution and current distribution. Lower expectation of $C_h$ over the set of samples indicates greater acceptance possibilities and indicated current distribution to be more similar as target one. After each step, by reintroducing the accept rate, the probability of a sample presenting in the set after $i^{th}$ selection should be

\begin{align}\begin{split}
    &p_{i+1}(\tau) = p(\text{selected}(\tau)|\tau) p_i (\tau)\\
    &= \frac{Z_i}{M Z_{tar}}\exp(C_i(\tau)-C_{tar}(\tau)) \frac{1}{Z_i}\exp(-C_i(\tau))\\
    & \propto \exp(-C_{tar}(\tau))\text{.}
\end{split}\end{align}

With preference-based sample selection, the trajectory distribution is expected to approach to the one under the target task finally. The convergence analysis will be provided in Section~\ref{analysis}.

\subsection{Enhanced Adversarial MaxEnt IRL for Distribution and Cost Learning}\label{MethoEAMaxEnt}
In the previous section, we introduce how the preference-based sample selection works in our task transfer framework. However, since the task transfer is an iterative process, we need to generate more samples with the same distribution as the selected samples set to keep it selectable by experts. Additionally, a cost function needs to be extracted from the selected demonstrations to optimize policies. With our enhanced Adversarial MaxEnt IRL, we can tackle these problems by learning the trajectory distribution and unbiased cost function simultaneously. 

Adversarial MaxEnt IRL~\cite{CONN_GAN_IRL} is a recently proposed GAN-based IRL algorithm that explicitly recovers the trajectory distribution from demonstrations. We enhance it to meet the requirements in our task transfer framework. Our enhancement is twofold:
\begin{itemize}
    \item Redefining the GAN from trajectory level to state-action pair level to extract a cost function that can be directly used for policy optimization.
    \item Although the GAN does not directly work on trajectory anymore, we prove that the generator can still be a sampler to the trajectory distribution of demonstrations.
\end{itemize}

We first briefly review the main ideas of Adversarial MaxEnt IRL. In this algorithm, demonstrations are supposed to be drawn from a Boltzmann distribution~\eqref{BoltzmannDist}, and the optimizing target can be regarded as Maximum Likelihood Estimation(MLE) over trajectory set $B$
\begin{align}\label{OptFuncGANIRL}
    \min_{\theta} L_{\text{cost}} = \mathbb{E}_{\tau \sim B}[-\log p_\theta(\tau)]\text{.}
\end{align}
The optimization in~\eqref{OptFuncGANIRL} can be cast into an optimization problem of a GAN~\cite{GAN,CONN_GAN_IRL}, where the discriminator takes the form as followed

\begin{align}\label{DFunc}
    D(\tau) = \frac{p(\tau)}{p(\tau) + G(\tau)} = \frac{\frac{1}{Z}\exp(-C(\tau))}{\frac{1}{Z}\exp(-C(\tau)) + G(\tau)}\text{.}
\end{align}

\citeauthor{CONN_GAN_IRL} showed that, when the model is trained to optimal, the generator $G$ will be an optimal sampler of the trajectory distribution $p(\tau) = \exp(-C(\tau))/Z$. However, we still cannot extract a closed-form cost function from the model. As a result, we enhanced it to meet our requirements.  

Since the cost function should be defined on each state-action pair, we first modified the input of the model in~\eqref{DFunc} from a trajectory to a state-action pair

\begin{align}\label{ModifiedDFunc}
    D(s, a) = \frac{\frac{1}{Z}\exp(-c(s, a))}{\frac{1}{Z}\exp(-c(s, a)) + G(s, a)}\text{.}
\end{align}
\noindent
The connection between the accurate cost $c(s, a)$ and outputs $D(s, a)$, $G(s,a)$ of GANs can be established

\begin{align}\label{estimateR}\begin{split}
&\tilde{c}(s, a) \coloneqq c(s, a) + \log Z\\
&=\log(1 - D(s, a)) - \log D(s, a) - \log G(s,a)\text{.}
\end{split}\end{align}

Here we define $\tilde{c}(s, a) = c(s, a) + \log Z$ as a cost estimator, while $c(s, a)$ is the accurate cost function. Since the partition function $Z$ is a constant while cost function is fixed, it will not affect the policy optimization, which means that $\tilde{c}$ can be directly integrated in common policy optimization algorithms as a unbiased cost function.

Notice that, after this modification, there will be several issues we need to address. Firstly, since the GAN is not defined on trajectory anymore, the equivalence between Guided Cost Learning and GAN training need to be re-verified. We will discuss it in Section~\ref{analysis}. Moreover, it is not straightforward whether $G(s, a)$ is a sampler to the distribution of demonstrations. 

We now show that when $G$ is trained to optimal, the distribution of trajectories sampled from it is exactly the distribution $p(\tau)$ of demonstrations:

\begin{assumption}\label{assumption_env_fixed}
The environment is stationary.
\end{assumption}

\begin{lemma}\label{lemma_expert_policy}
    Suppose that we have an expert policy $\pi_E(a|s)$ to produce demonstrations $B$, a trajectory $\tau = \{(s_0, a_0), (s_1, a_1), \cdots\}$ is sampled from $\pi_E$. Then $\tau$ will have the same probability as drawn from $p(\tau)$ if Assumption~\ref{assumption_env_fixed} holds ($p(\tau)$ is the trajectory distribution of $B$).
\end{lemma}

\begin{proof}
    We first introduce the environment model $p_e(s' | s, a)$ and the state distribution $p_s(s)$. In Reinforcement Learning, environment is basically a condition distribution over state transitions $(s', s, a)$. Thus the probability of a given trajectory $\tau = \{(s_0, a_0), (s_1, a_1) \cdots\}$ will be
    \begin{align}
        p(\tau) = p_s(s_0)\prod_{t = 0} \pi_E(a_t | s_t)p_e(s_{t+1} | s_t, a_t)\text{.}
    \end{align}
    
    Now we sample a trajectory $\tau$ with $\pi_E$ by executing roll-outs. Under Assumption~\ref{assumption_env_fixed}, the environment model $p_e$ for sampling $\tau$ from $\pi_E$ will be the same as sampling the demonstrations $B$, while $p_s(s) = \iint p_e(s'|s, a)\,ds\,da$ is also identical. Therefore, the probability of sampling $\tau$ can be derived as
    \begin{align}\label{sampling_prob}
        q(\tau) = p_s(s_0)\prod_{t = 0} \pi_E(a_t | s_t)p_e(s_{t+1} | s_t, a_t)\text{.}
    \end{align}
    It's obvious that $p(\tau) = q(\tau)$.
\end{proof}

\begin{lemma}\label{lemma_gan}
    \cite{GAN} The global minimum of the discriminator objective~\eqref{GANLOSS} is achieved when $p_G = p_{data}$.  
\end{lemma}

For a GAN defined on state-action level, with Lemma~\ref{lemma_gan}, $p_G = p_{data} = \pi_E$, $\pi_E$ is the expert policy for producing demonstrations. Then with Lemma~\ref{lemma_expert_policy}, it's obviously that the trajectory sampled with $G(s, a)$ will have the same density as $p(\tau)$, which means that $G(s, a)$ can still be a sampler to the trajectory distribution of demonstrations.

We formulate the minimization of generator loss as a policy optimization problem. We regard the unbiased cost estimator $\tilde{c}(s, a)$ as the cost function instead of $\mathcal{L}_G$ in~\eqref{GANLOSS}, and $G$ as a policy $\pi$. Thus the policy objective will be
\begin{align}\label{loss_pi}\begin{split}
    &\mathcal{L}_\pi = \mathbb{E}_{(s,a) \in B}\left[\log(1-D(s, a)) - \log D(s, a)\right] + \mathcal{H}(\pi)\\
    &\text{where}~ \mathcal{H}(\pi) = \mathbb{E}_{(s,a) \in B}\left[-\log\pi(a|s)\right]\text{.}
\end{split}\end{align}
This is quite similar to the generator objective used by GAIL~\cite{GAIL} but with an extra entropy penalty. We'll compare the performances of cost learning between our method and GAIL in Section~\ref{Exp}.

\subsection{Preference-based Task Transfer}

The entire task transfer framework is demonstrated in Algorithm~\ref{alg1}, which combines the hidden cost model for preference-based selection and enhanced Adversarial MaxEnt IRL for distribution and cost learning. With this framework, a well-trained policy on the basic task can be transferred to target task without accurate demonstrations or cost.

Comparing to Section~\ref{MethoEAMaxEnt}, we adopt a stop condition with $\epsilon$ and $M$ which indicates the termination of the loop, and an extra selection constraint which is observed to be helpful for stability in preliminary experiments. In practice, the parameters of $G_{\phi_i}$ and $D_{\omega_i}$ can be directly inherited from $G_{\phi_{i-1}}$ and $D_{\omega_{i-1}}$ when $i > 1$. Compare to initialize from scratch, this will converge faster in each iteration, while the results remain the same.

\begin{algorithm}[H]
\caption{Preference-based task transfer via Adversarial MaxEnt IRL}  
\begin{algorithmic}[1]
\REQUIRE ~~ \\
Demonstrations set $B_0$ on basic task.\\
Stop indicator $\epsilon$, maximum episode $M$.
Preference rules, or emulators which provides selection results.
\ENSURE ~~ \\
Transferred policy $\pi_{t}$.\\
~\newline
$i = 0$ \\
Initialize generator $G_{\phi_0}$, discriminator $D_{\omega_0}$; 
\REPEAT
\STATE $i \leftarrow i + 1$
\FOR{step $s$ in $\{$1, $\cdots$, N$\}$}
\STATE Sample trajectory $\tau$ from $G_{\phi_i}$; \ 
\STATE Update $D_{\omega_i}$ with binary classification error in~\eqref{GANLOSS} to tell demonstration $\tau^E \in B_{i-1}$ from sample $\tau$;
\STATE Update $G_{\phi_i}$ using any policy optimization method with respect to $\mathcal{L}_\pi$ in ~\eqref{loss_pi};
\ENDFOR  \\
\STATE Sampling with $G_{\phi_i}$, and collect $\tilde{B_i}$;
\STATE Query for preference to select trajectory in $\tilde{B_i}$ to obtain retained samples $B_i$, dropped samples $\overline{B}_i$, and guarantee $|\overline{B}_i|$ is no more than half of $|\tilde{B_i}|$;
\STATE Random sample $\beta |\overline{B}_i|$ trajectories from $\overline{B}_i$ and put them back into $B_i$;
\UNTIL{ $| \overline{B}_i | / | \tilde{B}_i| < \epsilon$  or $i = M$}
\label{Alog1} 
\RETURN $\pi_{t} \leftarrow G_{\phi_i}$
\end{algorithmic}
\label{alg1}
\end{algorithm}  

\subsection{Theoretical Analysis}\label{analysis}
In this section, we will discuss how can our framework learn the distribution from trajectories in each episode and finally transfer the cost function to target task. Remember the core part in our framework: Transferring the trajectory distribution from $p_0$ to $p_{tar}$. There is a finite loop in this process, during which we query for preference as $p_{\text{sel}}(\tau) \propto \exp \{-C_h(\tau)\}$ and improve the distribution $p_i$ for each episode $i$. If the distribution improves monotonically and the improvement can be maintained, we can guarantee the convergence of our method, which means that $p_{tar}$ can be learned. Then the cost function $c_i$ we learned together with $p_i$ will also approach to the cost for target task $c_{tar}$. This intuition is shown as following: 

\begin{proposition}\label{proposition1}
Given a finite set of trajectories $B$ sampled from distribution $p(\tau)$ and an expert with select probability~\eqref{selectProb}, the hidden cost over a trajectory $\mathbb{E}_{\tau \sim p}[-C_h(\tau)]$ is improved monotonically after each selection.
\end{proposition}

This proposition can be proved with some elementary 
derivations. Here we only provide the proof sketch. Since all the trajectories in $B$ are sampled from corresponding distribution $p(\tau)$, the expect cost can be estimated. Notice that we use a normalized select probability $p_{\text{sel}}(\tau) = \exp(-C_h(\tau))/Z$. Thus the estimations of expectation before and after the selection will be
\begin{align}\begin{split}
& \mathbb{E}_{\tau \sim p}[-C_h(\tau)] \approx \frac{1}{|B|}\sum\limits_{i = 0}^{|B|}(-C_h(\tau_i)) \\ 
& \mathbb{E}_{\tau \sim p'}[-C_h(\tau)] \approx \sum\limits_{i=0}^{|B|}p_{\text{sel}}(\tau_i)(-C_h(\tau_i))~/~\sum\limits_{i = 0}^{|B|}p_{\text{sel}}(\tau_i)\text{.}
\end{split}\end{align}

Obviously, trajectories after selection can not be seen as samples drawn from $p$, here we use $p'$, which can be regarded as an \textit{improved} $p$. Under linear expansion of cost, $\mathbb{E}_{\tau \sim p'}[-C_h(\tau)] \geq \mathbb{E}_{\tau \sim p}[-C_h(\tau)]$ can be proved. Thus the expect cost over a trajectory is improved monotonically.

Then we need to re-verify that whether the proposed state-action level GAN in our enhanced Adversarial MaxEnt IRL is still equivalent to Guided Cost Learning~\cite{GCL}:  
\begin{theorem}\label{theorem_gan_gcl_1}
Suppose we have demonstrations $B = \{\tau_0, \tau_1, \cdots \}$, a GAN with generator $G_\phi(s, a)$, discriminator $D_{\omega}(s, a)$. Then when the generator loss $\mathcal{L}_{G} = \mathbb{E}_{\tau \sim p}[\log (1- D_{\omega}(s, a)) - \log (D_{\omega}(s, a))]$ is minimized, the sampler loss in Guided Cost Learning~\cite{GCL} $\mathcal{L}_{sampler} = D_{KL}(q(\tau) ~||~ \exp (-C(\tau))/Z)$ is also minimized. $q(\tau)$ is the learned trajectory distribution, and $G_\phi$ is corresponding sampler.
\end{theorem}


Since the $\mathcal{L}_{sampler}$ is minimized along with $\mathcal{L}_{G}$, when the adversarial training ends, an optimal sampler of $q(\tau)$ can be obtained. Now we need to prove $B$ is drawn from $q(\tau)$:

\begin{theorem}\label{theorem_gan_gcl_2}
Under the same settings in Theorem~\ref{theorem_gan_gcl_1}, when the discriminator loss is minimized, the cost loss in Guided Cost Learning $\mathcal{L}_{cost} = \mathbb{E}_{\tau \sim B}[C_\theta(\tau)] + \mathbb{E}_{\tau \sim G}[\exp (-C_\theta(\tau))/q(\tau)]$ is also minimized. Thus the learned cost $C_\theta(\tau)$ is optimal for $B$. Refer to Theorem~\ref{theorem_gan_gcl_1}, $B$ is drawn from $q(\tau)$. 
\end{theorem}


In Theorem~\ref{theorem_gan_gcl_2}, MaxEnt IRL is regarded as a MLE of ~\eqref{BoltzmannDist}, while the unknown partition function $Z$ needs to be estimated. Therefore, training a state-action level GAN is still equivalent to maximizing the likelihood of trajectory distribution. Thus we can learn the optimal cost function and distribution under the current trajectory set $B$ at the same time. 

    With Proposition~\ref{proposition1}, we can start from an arbitrary trajectory distribution $p_0$ and trajectory set $B_0$ drawn from it. Then we can define a \textit{trajectory distribution iteration} as $p_{i+1}(\tau) \propto p_{i}(\tau)\exp\{-C_h(\tau)\}$~\cite{softQ}. Then expected hidden cost over a trajectory $\mathbb{E}_{\tau \sim p_i}[-C_h(\tau)]$ improves monotonically in each episode.
With Theorem~\ref{theorem_gan_gcl_1}, \ref{theorem_gan_gcl_2}, by strictly recovering the $improved$ distribution as $p_{i+1}$ from trajectory set (after selection), our algorithm can guarantee to maintain the improvement of expect cost over a trajectory to next episode. 
\begin{figure*}[t]
\centering
    \includegraphics[width=0.93\linewidth]{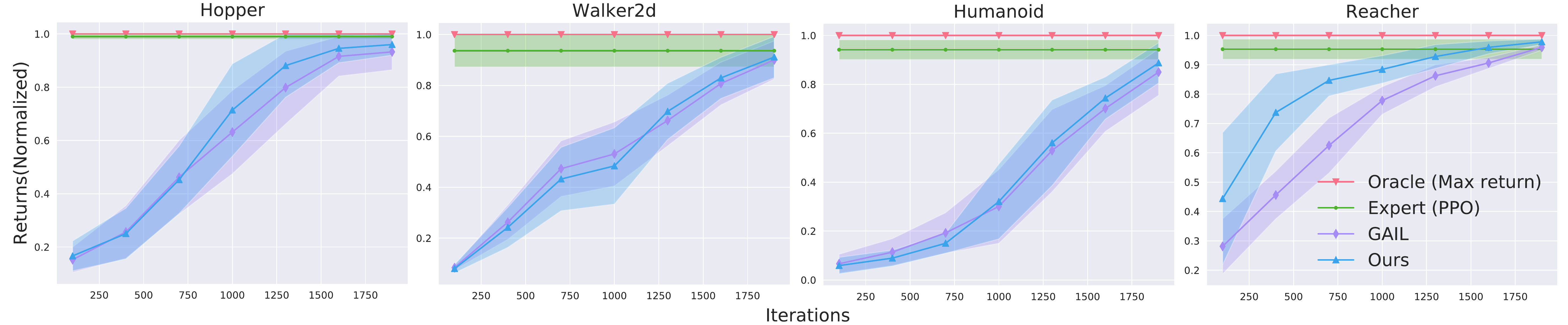}
    \caption{Results of distribution learning under four MuJoCo environments. Here the demonstrations are provided by an expert policy (PPO) under a known cost function. We compare the average cost value among trajectories generated by an \textit{oracle} (an ideal policy that always obtains maximum return), PPO (sample generator, acts as the expert), GAIL (a state-of-the-art IRL algorithm) and our distribution learning method. The results show that our method can finally achieve nearly the same performance as the expert. As we discussed in Section~\ref{ExpOverview}, we can verify that our method can learn the distribution from demonstrations.} 
    \label{ex1}
\end{figure*}

\begin{figure}[!ht]
\begin{center}
    \hfill\includegraphics[width=0.45\linewidth]{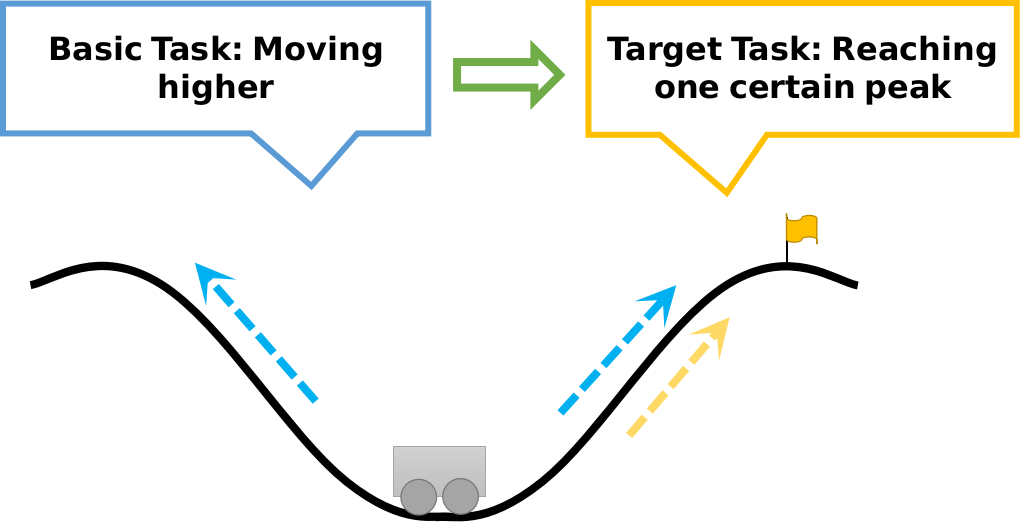}\hfill
    \includegraphics[width=0.45\linewidth]{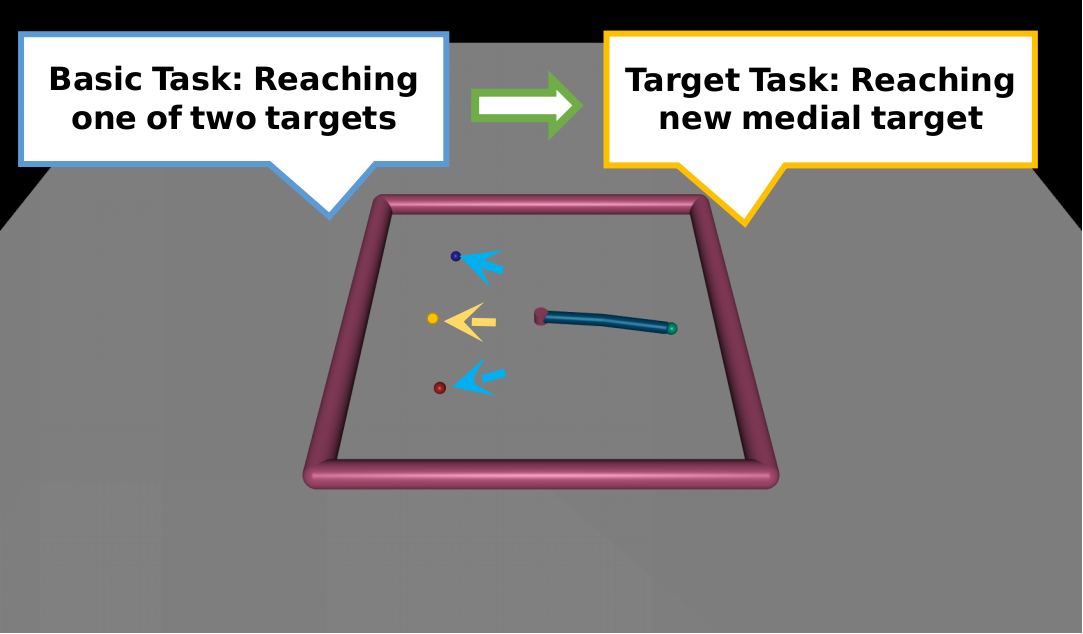}\hfill
    \includegraphics[width=1.0\linewidth]{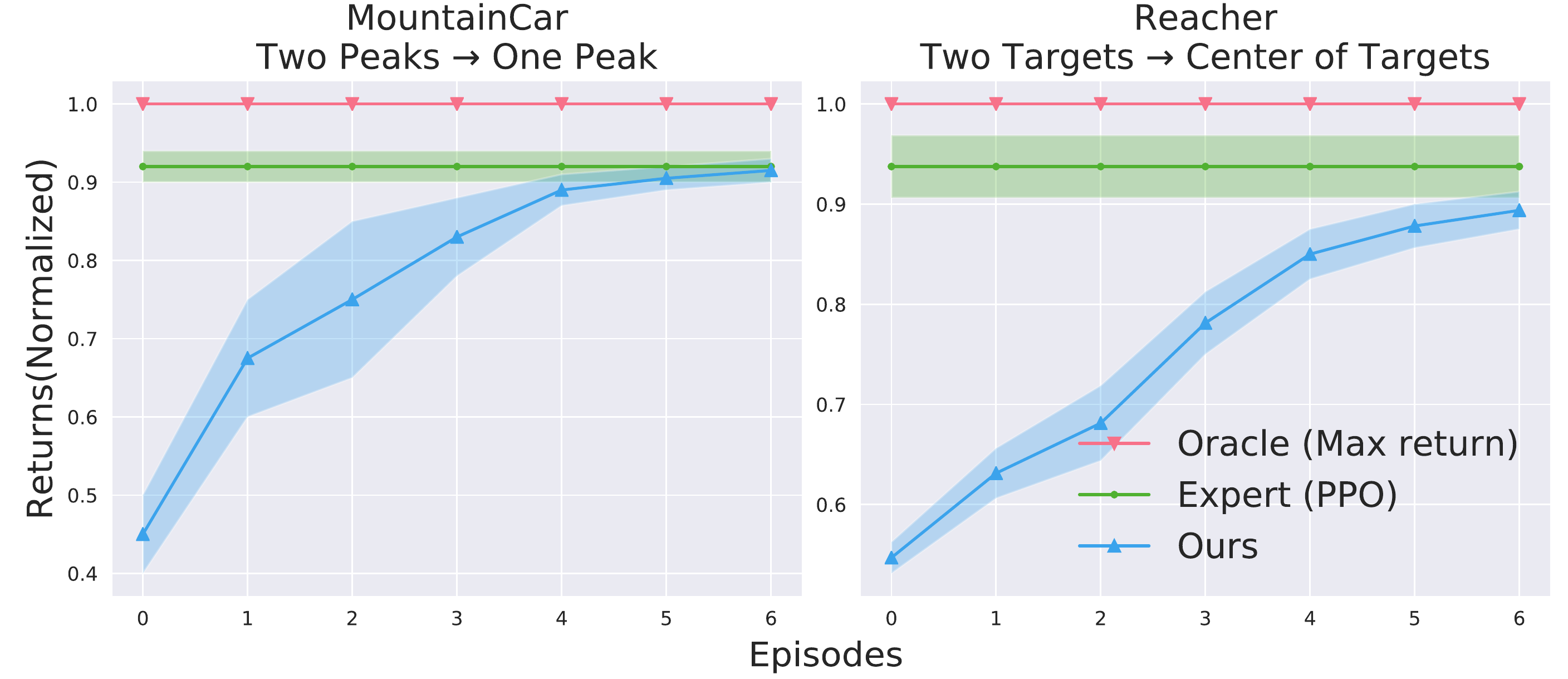}
\end{center}
   \caption{Results of cost learning and task transfer. We compare the average returns among an \textit{oracle} (an ideal policy that always obtains maximum return), an expert policy trained with the cost of target task, and our method. The results show that our algorithm can adapt to new task efficiently within $4 \sim 6$ episodes, and achieves nearly the same performance as the expert.}
\label{ex2}
\end{figure}

Under certain regularity conditions~\cite{softQ}, $p_i$ converges to $p_{\infty}$.
For trajectories that sampled from the target distribution $p_{tar}(\tau)$, their corresponding select probability~\eqref{selectProb} will approach to $1$. Thus $p_{tar}(\tau) = \exp(-C_h(\tau)-C_b(\tau))/Z$ can be a fixed point of this iteration when the iteration starts from $p_b(\tau) = \exp (-C_b(\tau))/Z$. Since all the non-optimal distribution can be improved this way, the learned distribution will converge to $p_{tar}(\tau)$ at infinity. As we have showed before, with a limited demonstrated trajectories $B$ sampled from arbitrary trajectory distribution $p(\tau)$, an optimal cost $c(a, s)$ can be extracted through our enhanced Adversarial MaxEnt IRL proposed in Section~\ref{MethoEAMaxEnt}. Therefore, the target cost $c_{tar}$ can also be learned from transferred distribution $p_{tar}$.

\section{Experiments}\label{Exp}

We evaluate our algorithm on several control tasks in MuJoCo~\cite{mujoco} physical simulator with pre-defined ground-truth cost function $c_b(s, a)$ on basic tasks and $c_{tar}(s, a)$ on target tasks in each experiments, $C_b(\tau)$ and $C_{tar}(\tau)$ are accumulated costs over trajectory $\tau$ for basic and target task respectively. All the initial demonstrations are generated by a well-trained PPO using $c_b$, and during the transfer process, preference is given by emulator with negative utility function (or hidden cost over a trajectory) $C_h = C_{tar} - C_b$. The select probability follows the definition in~\eqref{selectProb}. For performance evaluation, we use averaged return with respect to $c_{tar}(s, a)$ as the criterion.

\subsection{Overview}\label{ExpOverview}

In experiments, we mainly want to answer three questions:
\begin{enumerate}
\item During the task transfer procedure, can our method recover the trajectory distribution from demonstrations in each episode?
\item Starting from a basic task, can our method finally transfer to the target task and learn the cost function of it?
\item Under the same task transfer problem, can our method (based on preference only) obtain a policy with comparable performance, compared to other task transfer algorithms (based on accurate cost or demonstrations)?  
\end{enumerate}

To answer the first question, we need to verify the distribution learning part in our method functionally. Since our enhanced Adversarial MaxEnt IRL is built upon MaxEnt IRL, the recovered trajectory distribution can be reflected as a cost function, and the trajectories we learn from being generated by the optimal policy under that cost. Intuitively, given the expert trajectories $\tau_{\text{PPO}} $ generated by PPO and its corresponding cost $C_{tar}(\tau_{\text{PPO}})$, if we can train a policy which can generate $\tau$ with similar average $C_{tar}(\tau)$, we believe that the trajectory distribution can be recovered. 

To answer the second question, we evaluate the complete preference-based task transfer algorithm under some customized environments and tasks. In each environment, we transfer current policy under basic task to the target one. During the transfer process, expert preference (emulated by computer) is given as a selection result only, while any information of cost or selecting rule is unknown to the agent. We also train an expert policy with PPO and $c_{tar}(s, a)$ for comparison. In each episode, we generate $\tau_i$ using our learned policy and record $C_{tar}(\tau_i)$. If the average $C_{tar}(\tau_i)$ finally approaches to $C_{tar}(\tau_{\text{PPO}})$, we can verify that our method can learn the cost function of target task.

To answer the third question, we compare our method with MAML~\cite{MAML}, a task transfer algorithm requiring accurate $c_{tar}$. We use averaged cost on target task in each episode (we consider \textit{gradient step} in MAML the same as \textit{episode} in our method) for evaluation, to see whether the result of our method is comparable. 


\subsection{Environments and Tasks}
Here we outline some specifications of the environments and tasks in our experiments:
\begin{itemize}
    \item \textbf{Hopper, Walker2d, Humanoid and Reacher:}
    These environments and tasks are directly picked from OpenAI Gym~\cite{gym} without customization. Since they are only used for functionally verifying our distribution learning part and comparing with the original GAIL algorithm, there are no transfer settings.
    \item \textbf{MountainCar, Two Peaks $\rightarrow$ One Peak:}
    In this environment, there are two peaks for the agent to climb. The basic task is to make the vehicle higher, while the target task is climbing to a specified one.
    \item \textbf{Reacher, Two Targets $\rightarrow$ Center of Targets:}
    In this environment, the agent needs to control a 2-DOF manipulator to reach some specified targets. For the basic task, there will be two targets, and the agent can reach any of them, while in target task the agent is expected to reach the central position between the two targets.
    \item \textbf{Half-Cheetah, Arbitrary $\rightarrow$ Backward:}
    In this environment, the agent needs to control a multi-joint (6) robot to move forward or backward. The two directions are all acceptable in the basic task, while only moving backward is expected in the target task.
    \item \textbf{Ant, Arbitrary $\rightarrow$ Single:}
    This environment enhances the Half-Cheetah environment in two aspects: First, there will be more joints (8) to control; Second, the robot can move to arbitrary directions. In the basic task, any directions are allowed, while only one specified direction is expected in the target task.
\end{itemize}

\begin{figure}
\begin{center}
    \hfill\includegraphics[width=0.45\linewidth]{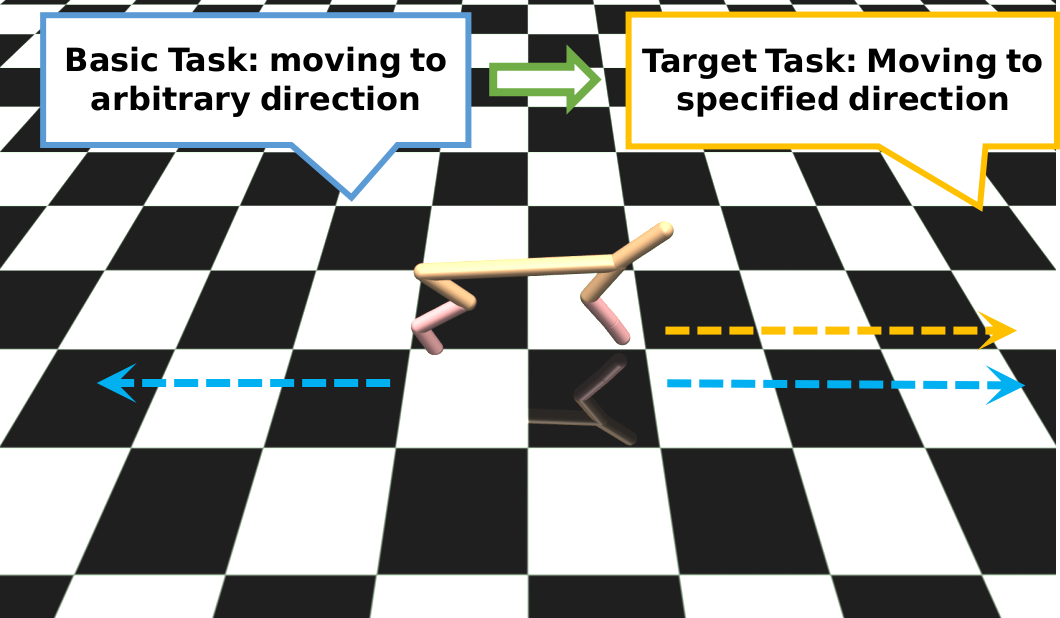}\hfill
    \includegraphics[width=0.45\linewidth]{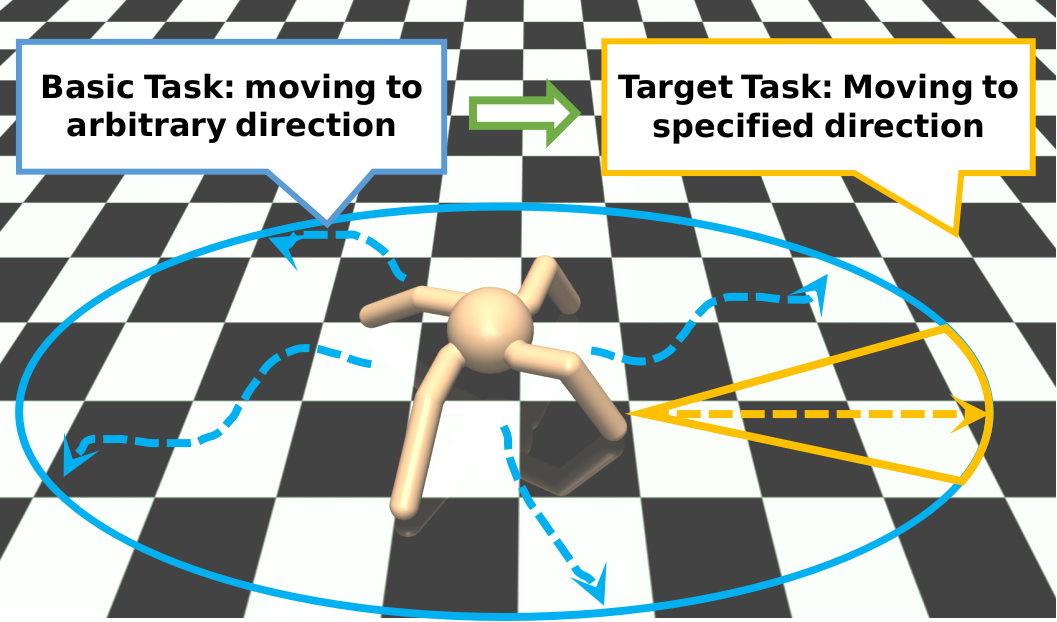}\hfill
    \includegraphics[width=1.0\linewidth]{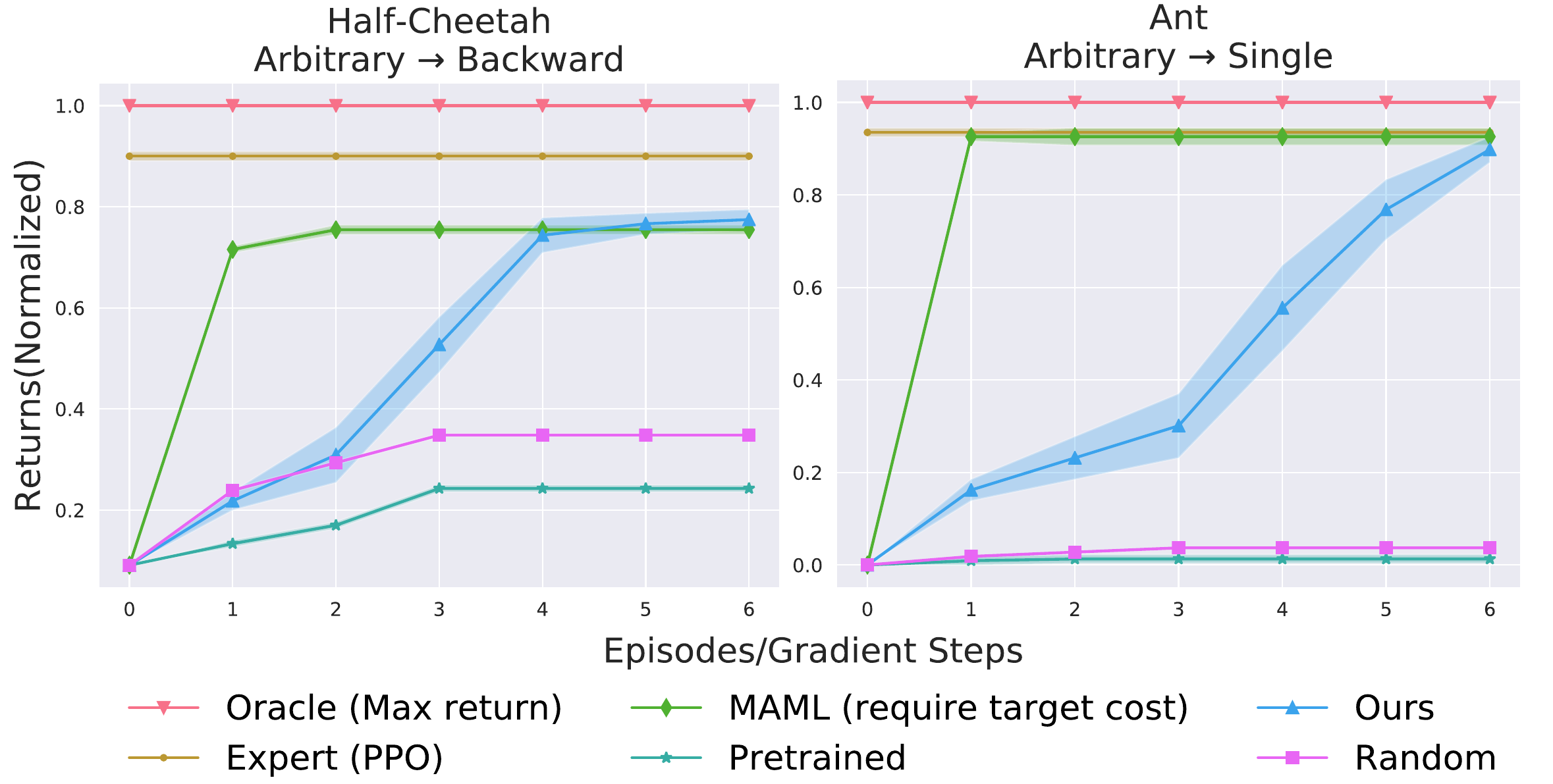}
\end{center}
   \caption{Results of comparison with other methods. We evaluate our algorithm under the transfer environments introduced by~\cite{MAML}. For the baselines, \textit{MAML} requires accurate $c_{tar}$ when transferring, \textit{Pretrained} means pre-training one policy from a basic task using Behavior Cloning~\cite{BC2} then fine-tuning. \textit{Random} means optimizing a policy from randomly initialized weights. The results show that our method can obtain a policy with comparable performance with MAML and other baselines.}
\label{ex3}
\end{figure}

\subsection{Distribution and Cost Learning}
We first concern the question whether our method can recover the trajectory distribution from demonstrations during the task transfer procedure. Experiment results are shown in Figure~\ref{ex1}. All the selected control tasks are equipped with high-dimensional continuous state and action spaces, which can be challenging to common IRL algorithms. We find that our method achieves nearly the same final performance as the expert (PPO) that provides the demonstrations, indicating that our method can recover the trajectory distribution. Also, comparing with other state-of-the-art IRL methods like GAIL, our method can learn a better trajectory distribution and a cost function more efficiently.

\subsection{Preference-based Task Transfer}
In Figure~\ref{ex2}, we demonstrate the transfer results on two environments. The transfer in Reacher environment is more difficult than MountainCar toy environment. The reason would be that the later one can be clustered easily since there are only two actual goals that a trajectory may reach, and the target goal (to reach one specified peak) is exactly one of them. In Reacher environment, although the demonstrations in the basic task still seem to be easily clustered, the target task cannot be directly derived from any of the clusters. In both two transfer experiments, the adapted policies produced by our algorithm show nearly the same performances as the experts that directly trained on these two target tasks. As the transferred policy is trained with the learned cost function, we can conclude that our algorithm can transfer to target task by learning the target cost function. In our experiments, we find that within less than 10 episodes and less than 100 querying number at each episode can sufficiently derive desired performance. Another potential improvement of our method is to apply some commutable rules to simulate the human selection and reduce the querying time. 

\subsection{Comparison with Other Methods}
We compare our method with some state-of-the-art task transfer algorithms including MAML~\cite{MAML}. Results are shown in Figure~\ref{ex3}. Half-Cheetah environment is pretty similar to MountainCar for the limited moving directions. However, its state and action space dimensions are much higher, which increase the difficulties for trajectory distribution and cost learning. Ant is the most difficult one among all the environments. Due to its unrestricted moving directions, the demonstrations on the basic task are highly entangled. The results illustrate that our method achieves comparable performances to those methods that require the accurate cost of the target task on the testing environments. Notice that, for some hard environment like Ant, our method may run for more episodes than MAML, since our algorithm only depends on preference, the results can still be convincing and impressive. 

\section{Conclusion}

In this paper, we present an algorithm that can transfer policies through learning the cost function on the target task with expert-provided preference selection results only. By modeling the preference-based selection as rejection sampling and utilizing enhanced Adversarial MaxEnt IRL for directly recovering the trajectory distribution and cost function from selection results, our algorithm can efficiently transfer policies from a related but not exactly-relevant basic task to the target one, while theoretical analysis on convergence can be provided at the same time. Comparing to other task transfer methods, our algorithm can handle the scenario in which acquiring the accurate demonstrations or cost functions from experts is inconvenient. Our results achieve comparable task transfer performances to other methods which depend on accurate costs or demonstrations. Future work could focus on the quantitative evaluation of the improvement on the transferred cost function. Also, the upper bound on the sum of total operating episodes could be analyzed. 

\section*{Acknowledgment}
This research work was jointly supported by the Natural Science Foundation great international cooperation project (Grant No:61621136008) and the National Natural Science Foundation of China (Grant No:61327809). Professor \textbf{Fuchun Sun(fcsun@tsinghua.edu.cn)} is the corresponding author of this paper and we would like to thank Tao Kong, Chao Yang and Professor Chongjie Zhang for their generous help and insightful advice.\\

\small{
\bibliographystyle{aaai}
\bibliography{reference}
}

\end{document}